\documentclass[11pt,journal,compsoc,onecolumn]{IEEEtran}
 \usepackage{cite}

\ifCLASSINFOpdf
\else
\fi
\ifCLASSOPTIONcompsoc
 \usepackage[font=normalsize,labelfont=sf,textfont=sf]{subfig}
\else
 \usepackage[font=footnotesize]{subfig}
\fi
\ifCLASSOPTIONcompsoc
 \usepackage[font=normalsize,labelfont=sf,textfont=sf]{subfig}
\else
 \usepackage[font=footnotesize]{subfig}
\fi

\usepackage{amsmath}
\usepackage{color}
 \usepackage{amssymb}
\usepackage{amsfonts}
\usepackage{dsfont}
\usepackage{epsfig}

\title{Canonical Correlation Analysis for Misaligned Satellite Image Change Detection}
\author{ Hichem Sahbi \\ $ $ \\ {CNRS, LIP6 Lab, Sorbonne University, Paris}}

\begin{document}

\maketitle             

\begin{abstract}
  Canonical correlation analysis (CCA) is a statistical learning method that seeks to build view-independent latent representations from  multi-view data. This method has been successfully applied to several pattern analysis tasks such as image-to-text mapping and view-invariant object/action recognition. However, this success is highly dependent on the quality of data pairing (i.e., alignments) and mispairing adversely affects the generalization ability of the learned CCA representations.\\
  In this paper, we address the issue of alignment errors using a new variant of canonical correlation analysis referred to as alignment-agnostic (AA) CCA. Starting from erroneously paired data taken from different views, this CCA finds transformation matrices by optimizing a  constrained  maximization problem that mixes a data correlation term with context regularization; the particular design of these two terms mitigates the effect of alignment errors when learning the CCA transformations. Experiments conducted on multi-view tasks, including multi-temporal satellite image change detection, show that our AA CCA method is highly effective and resilient to mispairing errors. 
\end{abstract}


\newtheorem{proposition}{Proposition}
\newtheorem{proof}{Proof}

\section{Introduction}
\indent Several tasks in computer vision and neighboring fields require {\it labeled} datasets in order to build effective statistical learning models. It is widely agreed that the accuracy of these models  relies substantially on the availability of large labeled training sets. These sets require a tremendous human annotation effort and are thereby  very expensive for many large scale classification problems including  image/video-to-text  (a.k.a captioning)\cite{krizhevsky2012imagenet,sahbi2013cnrs,russakovsky2015imagenet,boujemaa2004visual,wang2013directed}, multi-modal information retrieval \cite{tollari2008comparative,ferecatu2008telecomparistech}, multi-temporal change detection \cite{hussain2013change,bourdis2012spatio}, object recognition and segmentation~\cite{Postadjian17,sahbi2002coarse,li2011superpixel}, etc. The current trend in machine learning, mainly with the data-hungry deep models \cite{krizhevsky2012imagenet,russakovsky2015imagenet,he2016deep,goodfellow2016deep,jiu2017nonlinear}, is to bypass supervision, by making the training of these models totally unsupervised \cite{erhan2010does}, or at least weakly-supervised using: fine-tuning \cite{yosinski2014transferable}, self-supervision \cite{doersch2017multi}, data augmentation and game-based models \cite{richter2016playing}. However, the hardness of collecting annotated datasets does not {\it only} stem from assigning accurate labels to these data, but also from aligning them; for instance, in the neighboring field of machine translation, successful training models require accurately aligned bi-texts (parallel bilingual training sets), while in satellite image change detection, these models require accurately georeferenced and registered satellite images. This level of requirement, {\it both} on the accuracy of labels and {\it their alignments}, is clearly hard-to-reach; alternative models, that skip the sticky alignment requirement, should be preferred. \\ 

\indent  Canonical correlation analysis (CCA) \cite{hotelling1936relations,anderson1958introduction,hardoon2004canonical,sahbi2017interactive} is one of the statistical learning models that require accurately aligned (paired) multi-view data\footnote{Multi-view data stands for input data described with multiple modalities such as documents described with text and images.}; CCA finds -- for each view -- a transformation matrix that maps data from that view to a view-independent (latent) representation such that aligned data obtain  highly correlated latent representations. Several extensions of CCA have been introduced in the literature including nonlinear (kernel) CCA \cite{melzer2003appearance},  sparse CCA \cite{hardoon2011sparse,witten2009extensions,zhang2013binary}, multiple CCA \cite{via2007learning}, locality preserving and instance-specific CCA \cite{sun2007locality,zhai2015instance}, time-dependent CCA \cite{yger2012adaptive} and other unified variants (see for instance \cite{de2009unification,sun2013canonical}); these methods have been applied to several pattern analysis tasks such as image-to-text \cite{sun2011canonical}, pose estimation \cite{melzer2003appearance,sun2007locality} and object recognition \cite{haghighat2017low}, multi-camera activity correlation \cite{ferecatu2009multi,loy2009multi} and motion alignment \cite{zhou2009canonical,zhou2016generalized} as well as  heterogeneous sensor data classification \cite{kim2009canonical}.\\

\indent The success of all the aforementioned CCA approaches is highly {\it dependent} on the accuracy of  alignments between  multi-view data.  In practice, data are subject to misalignments (such as registration errors in satellite imagery) and sometimes completely unaligned (as in muti-lingual documents) and this skews the learning of CCA. Excepting a few attempts -- to handle temporal deformations in monotonic sequence datasets \cite{fischer2007time} using canonical time warping~\cite{zhou2016generalized} (and its deep extension \cite{trigeorgis2017deep}) -- none of these existing CCA variants address alignment errors for non-monotonic datasets\footnote{Non-monotonic stands for datasets without a ``unique'' order (such as patches in images).}.  Besides CCA, the  issue of data alignment has been approached, in general, using manifold alignment~\cite{ham2005semisupervised,lafon2006data,wang2008manifold}, Procrustes analysis \cite{luo2002iterative} and source-target domain adaption ~\cite{feuz2017collegial} but none of these methods consider resilience to misalignments as a part of CCA design (which is the main purpose of our contribution in this paper). Furthermore, these data alignment solutions rely on a strong ``apples-to-apples'' comparison hypothesis ({\it that data taken from different views have similar structures}) which does not always hold especially when handling datasets with heterogeneous views (as text/image data and multi-temporal or multi-sensor satellite images). Moreover, even when data are globally well (re)aligned, some {\it residual} alignment errors are difficult to handle (such as parallax in multi-temporal satellite imagery) and harm CCA (as shown in our experiments). \\

\indent In this paper, we introduce a novel CCA approach that handles misaligned data; i.e.,  it does not require any preliminary step of accurate data alignment. This is again very useful for different applications where aligning data is very time demanding or when data are taken from multiple sources (sensors, modalities, etc.) which are intrinsically misaligned\footnote{Satellite images -- georeferenced with the Global Positioning System (GPS) -- have localization errors that may reach 15 meters in some geographic areas. On high resolution satellite images (sub-metric resolution) this corresponds to alignment errors/drifts that may reach 30 pixels.}. The benefit of our approach is twofold; on the one hand, it models the uncertainty of alignments using a new data correlation term and on the other hand, modeling alignment uncertainty allows us to use not only decently aligned data (if available) when learning CCA, but also the unaligned ones. In sum, this approach can be seen as an extension of CCA on unaligned sets compared to standard CCA (and its variants) that operate only on accurately aligned data. Furthermore, the proposed method is  as efficient as standard CCA and its computationally complexity grows w.r.t the dimensionality (and not the cardinality) of data, and this makes it very suitable for large datasets. \\

\noindent Our CCA formulation is based on the optimization of a constrained objective function that combines two terms; a correlation criterion and a context-based regularizer. The former maximizes a weighted correlation between data  with a high cross-view similarity  while the latter makes this weighted correlation   high for data  whose neighbors have high correlations too (and vice-versa). We will show that optimizing this constrained maximization problem is equivalent to solving an iterative generalized eigenvalue/eigenvector decomposition; we will also show that the solution of this iterative process converges to a fixed-point.  Finally, we will illustrate the validity of our CCA formulation  on different challenging problems including change detection both on {\it residually and strongly} misaligned multi-temporal satellite images; indeed, these images are subject to  alignment errors due to the hardness of image registration under challenging conditions, such as occlusion and parallax. \\

\indent The rest of this paper is organized as follows; section~\ref{section2} briefly reminds the preliminaries in canonical correlation analysis, followed by our main contribution: a novel alignment-agnostic CCA, as well as some theoretical results about the convergence of the learned CCA transformation to a fixed-point (under some constraints on the parameter that weights our regularization term). Section~\ref{section3} shows the validity of our method both on synthetic toy data as well as real-world problems namely satellite image change detection. Finally, we conclude the paper in section~\ref{section4} while providing possible extensions for a future work.  

\def\X{{\cal X}}
\def \S{{\cal I}}
\def \x{\mathbf{x}}
\def \u{\mathbf{u}}
\def \v{\mathbf{v}}
\def \y{\mathbf{y}}

\section{Canonical Correlation Analysis}\label{section2}

Considering the input spaces $\X_r$ and $\X_t$ as two sets of images taken from two modalities; in satellite imagery, these modalities could be two different sensors, or the same sensor at two different instants, etc.  Denote  $\S_r=\{\u_i\}_i$, $\S_t=\{\v_j\}_j$  as two subsets of $\X_r$ and $\X_t$ respectively; our goal is learn a transformation between $\X_r$ and  $\X_t$ that assigns, for a given $\u \in \X_r$, a sample $\v \in \X_t$. The learning of this transformation usually requires accurately {\it paired} data in $\X_r \times \X_t$ as in CCA. 

\def \Pr{{\textbf{P}_r}}
\def \Pt{{\textbf{P}_t}}
\def \Prt{{\textbf{P}'_r}}
\def \Ptt{{\textbf{P}'_t}}

\def \A{\textbf{A}}
\def \B{\textbf{B}}
\def \C{\textbf{C}}
\def \I{\textbf{I}}
\def \tr{{\textrm{tr}}}

\subsection{Standard  CCA}\label{standard}

Assuming centered data in $\S_r$, $\S_t$, standard CCA (see for instance \cite{hardoon2004canonical}) finds two projection matrices that map aligned data in $\S_r \times \S_t$ into a latent space while maximizing their correlation. Let $\Pr$, $\Pt$ denote these projection matrices which respectively correspond to  reference  and  test  images.   CCA  finds  these  matrices  as $(\Pr,\Pt)=\arg\max_{\A,\B} \tr(\A'\C_{rt} \B)$, subject to $\A'\C_{rr}\A=\I_u$, $\B'\C_{tt} \B =\I_{v}$; here $\I_{u}$ (resp. $\I_{v}$) is the $d_u\times d_u$ (resp. $d_v\times d_v$)  identity matrix, $d_u$ (resp. $d_v$) is the dimensionality of data in $\X_r$ (resp. $\X_t$),  $\A'$ stands for transpose of $\A$, $\tr$ is the trace, $\C_{rt}$ (resp. $\C_{rr}$, $\C_{tt}$) correspond to inter-class (resp. intra-class) covariance matrices of data in $\S_r$, $\S_t$, and equality constraints control the effect of scaling on the solution.  One can show that problem above is equivalent to solving the eigenproblem $\C_{rt}\C_{tt}^{-1}\C_{tr}\Pr =\gamma^2 \C_{rr}\Pr$ with $\Pt=\frac{1}{\gamma} \C_{tt}^{-1}\C_{tr} \Pr$. In practice, learning these two transformations requires ``paired'' data in $\S_r \times \S_t$, i.e., aligned data. However, and as will be shown through this paper, accurately paired data are not always available (and also expensive), furthermore the cardinality of $\S_r$ and $\S_t$ can also be different, so one should adapt CCA in order to learn transformation between data in $\S_r$ and $\S_t$ as shown subsequently. 

\def \bU {\textbf{U}}
\def \bV {\textbf{V}}
\def \D {\textbf{D}}

\subsection{Alignment Agnostic CCA}
We introduce our main contribution: a novel alignment agnostic CCA approach. Considering  $\{(\u_i,\v_j)\}_{ij}$ as a subset of $\S_r \times \S_t$ (cardinalities of $\S_r$, $\S_t$ are not necessarily equal), we propose to find the transformation matrices $\Pr$, $\Pt$ as
\begin{equation}\label{of0}
  \begin{array}{ll}  
\displaystyle  \max_{\Pr,\Pt}  &  \tr (\bU' \Pr \Ptt \bV \D) \\
    \textrm{s.t.}   & \Prt \C_{rr} \Pr = \I_{u}  \ \ \textrm{and} \ \  \Ptt \C_{tt}  \Pt = \I_{v},
                      \end{array}
\end{equation} 
the non-matrix form of this objective function is given subsequently. In this constrained maximization problem, $\bU$, $\bV$ are two matrices of data in $\S_r$, $\S_t$ respectively, and $\D$ is an (application-dependent) matrix with its given entry $\D_{ij}$ set to the cross affinity or the likelihood that a given data $\u_i \in \S_r$ aligns with  $\v_j \in \S_t$ (see section~\ref{pairing} about different setting of this matrix). This definition of $\D$, together with  objective function~(\ref{of0}), make CCA {\it alignment agnostic}; indeed, this objective function (equivalent to $\sum_{i,j} \langle \Prt \u_i,\Ptt \v_j \rangle \D_{ij}$) aims to maximize the correlation between  pairs (with a high cross affinity of alignment) while it also minimizes the correlation between pairs with small cross affinity. For a particular setting of $\D$, the following proposition provides a {\it special case.} \\

\begin{proposition}
\normalfont 
  provided that $|\S_r|=|\S_t|$ and $\forall \u_i \in \S_r$, $\exists ! \v_j \in \S_t$ such that $\D_{ij}= 1$;  the constrained maximization problem (\ref{of0}) implements standard CCA. 
\end{proposition} 
\begin{proof}
  \normalfont
considering the non-matrix form of (\ref{of0}), we obtain
{\normalfont  
  \begin{equation}\label{of00}
  \begin{array}{lll}  
    \displaystyle    \tr (\bU' \Pr \Ptt \bV \D) & = & \displaystyle  \sum_{i,j} \langle \Prt \u_i,\Ptt \v_j \rangle \D_{ij},
  \end{array}
\end{equation}
considering a particular order of $\S_t$ such that each sample $\u_i$ in $\S_r$ aligns with a unique $\v_i$ in $\S_t$ we obtain
\begin{equation}\label{of000}
\begin{array}{lll}  

          \displaystyle    \tr (\bU' \Pr \Ptt \bV \D)    & = &\displaystyle  \sum_{i,j} \langle \Prt \u_i,\Ptt \v_j \rangle \mathds{1}_{\{i=j\}} \\
    & = &\displaystyle  \sum_{i} \langle \Prt \u_i,\Ptt \v_i \rangle  \\
                               &  =  &  \tr(\displaystyle  \Prt   (\sum_{i}  \u_i \v'_i ) \Pt) \\
       &  = & \tr(\Prt \C_{rt} \Pt),
  \end{array}
\end{equation}
}
\noindent with  $\C_{rt}$ being the inter-class covariance matrix and $\mathds{1}_{\{\}}$  the indicator function. Since the equality constraints (shown in section~\ref{standard}) remain unchanged, the constrained maximization problem~(\ref{of0}) is strictly equivalent to standard CCA for this particular $\D$ $\Box$
\end{proof} 
\indent This particular setting of $\D$ is relevant only when data are accurately paired  and also when $\S_r$, $\S_t$ have the same cardinality. In practice, many problems involve unpaired/mispaired datasets with different cardinalities; that's why $\D$ should be relaxed using affinity between multiple pairs (as discussed earlier in this section) instead of using strict alignments. With this new CCA setting, the learned transformations $\Pt$, and $\Pr$ generate latent data representations  $\phi_t(\v_i)=\Ptt \v_i$, $\phi_r(\u_j)=\Prt \u_j$  which align  according to $\D$ (i.e., $\| \phi_r(\v_i)-\phi_t(\u_j)\|_2$ decreases if $\D_{ij}$ is high and vice-versa). However, when multiple entries $\{\D_{ij}\}_j$ are high for a given $i$, this may produce noisy correlations between the learned latent representations and may impact their discrimination power (see also experiments). In order to mitigate this effect, we also consider context regularization~\cite{sahbi2008context}.
\subsection{Context-based regularization}\label{context}
For each data $\u_i \in \S_r$, we define a (typed) neighborhood system $\{{\cal N}_c(i)\}_{c=1}^C$ which corresponds to the typed neighbors of $\u_i$ (see section~\ref{pairing} for an example). Using  $\{{\cal N}_c(.)\}_{c=1}^C$, we consider for each $c$ an intrinsic adjacency matrix ${\bf W}^c_u$  whose $(i,k)^{\textrm{th}}$ entry is set as ${\bf W}^c_{u,i,k} \propto \mathds{1}_{\{ k \in {\cal N}_c(i)\}}$. Similarly, we define  the matrices $\{{\bf W}^c_v\}_c$ for data $\{\v_j\}_j\in \S_t$; extra details about the setting of these matrices are again given in experiments. \\
\indent Using the above definition of $\{{\bf W}^c_u\}_c$, $\{{\bf W}^c_v\}_c$, we add an extra-term in the objective function (\ref{of0}) as  
\begin{equation}\label{ref11}
  \begin{array}{l}
\displaystyle \max_{\Pr,\Pt}  \  \textrm{tr}(\bU' \Pr \Ptt \bV \D)  + \beta  \displaystyle \sum_{c=1}^C\displaystyle \textrm{tr}\big( \bU' \Pr \Ptt \bV {\bf W}^c_{v} \bV' \Pt \Prt \bU {\bf W}^{c '}_{u}\big) \\ 
{\textrm{s.t.}}    \ \ \ \ \ \         \Prt \C_{rr} \Pr \ = {\bf I}_{u}  \ \ \ \ \ \textrm{and}   \ \ \  \ \         \Ptt \C_{tt} \Pt \  \ = {\bf I}_{v}.
\end{array}
\end{equation}
The above right-hand side term is equivalent to $$\small \beta \sum_c \sum_{i,j} \langle \Prt \u_i, \Ptt \v_j\rangle  \sum_{k,\ell} \langle \Prt \u_k, \Ptt \v_\ell\rangle {\bf W}^c_{u,i,k} {\bf W}^c_{v,j,\ell}$$ the latter corresponds to a neighborhood (or context) criterion which  considers that a high value of the correlation $\langle \Prt \u_i, \Ptt \v_j\rangle$,  in the learned latent space, should  imply  high correlation values in the neighborhoods $\{{\cal N}_c(i) \times {\cal N}_c(j)\}_c$. This term (via $\beta$) controls the sharpness of the correlations (and also the discrimination power) of the learned latent representations (see example in Fig.~\ref{fig:beta}). Put differently, if a given $(\u_i,\v_j)$ is surrounded by highly correlated pairs, then the correlation between $(\u_i,\v_j)$ should be maximized and vice-versa (see also~\cite{sahbi2010context,chetverikovSK05}). 

\subsection{Optimization}
Considering Lagrange multipliers for the equality constraints in Eq.~(\ref{ref11}), one may show that optimality conditions (related to the gradient of Eq.~(\ref{ref11}) w.r.t  $\Pr$, $\Pt$ and the Lagrange multipliers) lead to the following generalized eigenproblem
\begin{equation}\label{ep01}
\begin{array}{lll}
{\bf K}_{rt}  \C_{tt}^{-1}  {\bf K}_{tr} \Pr &= &\gamma^2 \C_{rr} \Pr \\
\textrm{with}  \ \ \ \ \ \Pt  & =& \frac{1}{\gamma} \  \C_{tt}^{-1}  {\bf K}_{tr} \Pr,
\end{array}
\end{equation}
\noindent here  ${\bf K}_{tr}={\bf K}_{rt}'$ and 
\begin{equation}\label{ep001}
\begin{array}{lll}
{\bf K}_{tr} =  \bV \D \bU' & + & \beta \sum_c \bV {\bf W}^c_v \bV' \Pt \Prt \bU {\bf W}^{c'}_u \bU' \\ 
           & +  & \beta \sum_c  \bV {\bf W}^{c'}_v \bV' \Pt \Prt \bU {\bf W}^c_u  \bU'.
\end{array}
\end{equation}

\noindent In practice, we solve the above eigenproblem iteratively. For each iteration $\tau$, we fix $\Pr^{(\tau)}$, $\Pt^{(\tau)}$ (in ${\bf K}_{tr}$, ${\bf K}_{rt}$) and we find the subsequent projection matrices  $\Pr^{(\tau+1)}$, $\Pt^{(\tau+1)}$ by solving Eq.~(\ref{ep01}); initially, $\Pr^{(0)}$, $\Pt^{(0)}$ are set using projection matrices of standard CCA. This process continues till a fixed-point is reached.  In practice,  convergence to a fixed-point is observed in less than five iterations.
 \begin{proposition}
\normalfont 
   let $\|.\|_1$ denote the entry-wise $L_1$-norm and ${\bf 1}_{\tiny vu}$ a $d_v\times d_u$  matrix of ones. Provided that the following inequality holds
\begin{equation}\label{loose} 
\beta < {\gamma_{\min}} \times  \bigg( \sum_c \big\| {\bf E}_c \ {\bf 1}_{\tiny vu} \ {\bf F}_c' \big\|_1 +  \sum_c \big\| {\bf G}_c \ {\bf 1}_{\tiny vu} \ {\bf H}_c' \big\|_1\bigg)^{-1}
\end{equation} 
\noindent with  $\gamma_{\min}$ being a lower bound of the positive eigenvalues of (\ref{ep01}), ${\bf E}_c=\bV {\bf W}^c_v \bV' \C_{tt}^{-1}$,  $ {\bf F}_c=\bU {\bf W}^{c}_u \bU' \C_{rr}^{-1}$,  ${\bf G}_c=\bV {\bf W}^{c'}_v \bV' \C_{tt}^{-1}$ and ${\bf H}_c=\bU {\bf W}^{c'}_u \bU' \C_{rr}^{-1}$; the problem in (\ref{ep01}), (\ref{ep001})  admits a unique solution $\tilde{\bf P}_r$, $\tilde{\bf P}_t$ as the eigenvectors of
\begin{equation}
\begin{array}{lll}\label{eq1111}
\tilde{\bf K}_{rt}  \C_{tt}^{-1}  \tilde{\bf K}_{tr} \Pr &= &\gamma^2 \C_{rr} \Pr \\
 \ \ \ \ \ \ \ \ \ \ \ \ \ \ \ \ \Pt  & =& \frac{1}{\gamma} \  \C_{tt}^{-1}  \tilde{\bf K}_{tr} \Pr,
\end{array}
\end{equation}
\noindent with $\tilde{\bf K}_{tr}$ being the limit of  
\begin{equation}
\begin{array}{lll}\label{rec0}

{\bf K}_{tr}^{(\tau+1)} = \Psi\big({\bf K}_{tr}^{(\tau)}\big),\\ 
\end{array}
\end{equation}
\noindent and $\Psi: \mathbb{R}^{d_v\times d_u}  \rightarrow \mathbb{R}^{d_v\times d_u}$  is given as  
\begin{equation}\label{rec}
\begin{array}{lll}

\Psi({\bf K}_{tr}) = \displaystyle  \bV \D \bU' & + & \beta \sum_c \bV {\bf W}^c_v \bV' \Pt \Prt \bU {\bf W}^{c'}_u \bU' \\ 
           & +  & \beta \sum_c  \bV {\bf W}^{c'}_v \bV' \Pt \Prt \bU {\bf W}^c_u  \bU',
\end{array}
\end{equation}
\noindent with $\Pt$, $\Pr$,  in (\ref{rec}), being functions of ${\bf K}_{tr}$ using (\ref{ep01}).
Furthermore, the matrices ${\bf K}_{tr}^{(\tau+1)}$ in (\ref{rec0}) satisfy the convergence property 
\begin{equation} 
\big\|{\bf K}_{tr}^{(\tau+1)} - \tilde{\bf K}_{tr}\big\|_1 \leq L^\tau\big\|{\bf K}_{tr}^{(\tau+1)} - {\bf K}_{tr}^{(0)}\big\|_1, 
\end{equation}
with $L=\frac{\beta}{\gamma_{\min}} \big(\sum_c\big\|  {\bf E}_c \ {\bf 1}_{\tiny vu} \ {\bf F}_c' \big\|_1 + \sum_c  \big\|{\bf G}_c \ {\bf 1}_{\tiny vu} \ {\bf H}_c' \big\|_1\big)$.
\end{proposition}
\begin{proof}
  see appendix 
\end{proof} 
Note that resulting from the extreme sparsity of the {\it typed} adjacency matrices  $\{{\bf W}^c_u\}_c$, $\{{\bf W}^c_v\}_c$, the upper bound about $\beta$ (shown in the sufficient condition in Eq.~\ref{loose}) is loose, and easy to satisfy; in practice,  we observed convergence for all the values of $\beta$ that were tried in our experiments (see the x-axis of Fig.~\ref{fig:beta}).  

\section{Experiments}\label{section3}
In this section, we show the performance of our method both on synthetic  and real datasets.  The goal is to show the extra gain brought when using our  alignment agnostic (AA) CCA approach  against standard CCA and other variants.

\subsection{Synthetic Toy Example} 

In order to show the strength of our AA CCA method, we first illustrate its performance on a 2D toy example.  We consider 2D data sampled from an ``arc'' as shown in Fig.~\ref{fig:toy}(a); each sample is endowed with an RGB color feature vector which depends on its curvilinear coordinates in that ``arc''. We duplicate this dataset using a 2D rotation (with an angle of $180^0$) and we add a random perturbation field (noise) both to the color features and the 2D coordinates (see Fig.~\ref{fig:toy}).  Note that accurate ground-truth pairing is available but, of course, not used in our experiments. \\
We apply our AA CCA (as well as standard CCA) to these data, and we show alignment results; this 2D toy example is very similar to the subsequent real data task as the goal is to find for each sample in the original set, its  correlations and its realignment with the second set.  From Fig.~(\ref{fig:toy}), it is clear that standard CA fails to produce accurate results when data is contaminated with random perturbations and alignment errors, while our AA CCA approach successfully realigns the two sets (see again details in Fig.~\ref{fig:toy}).

\begin{figure*}[t]
\begin{center}
\centering
\begin{tabular}[b]{c}
  \includegraphics[width=0.33\linewidth]{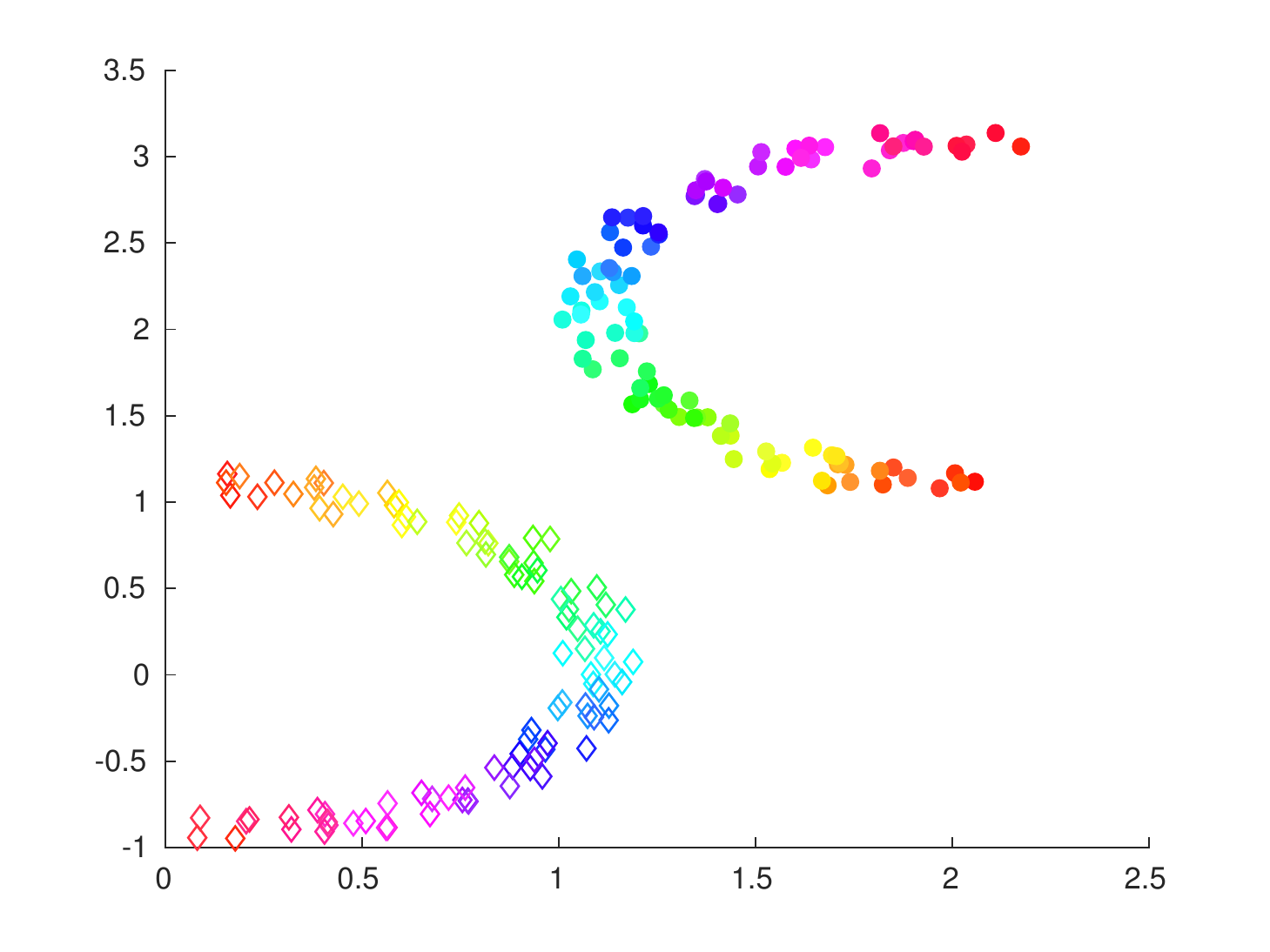}\includegraphics[width=0.33\linewidth]{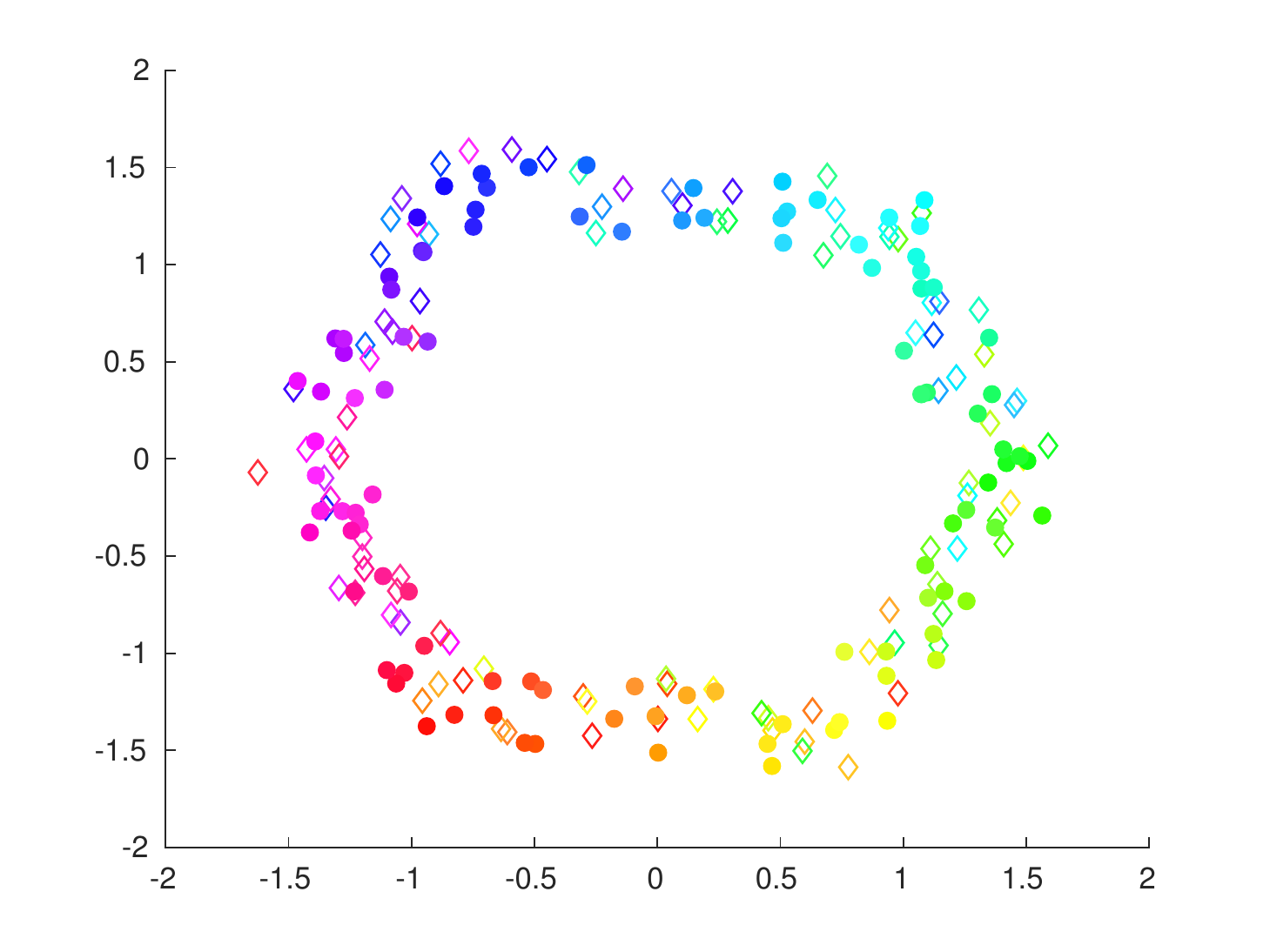}\includegraphics[width=0.33\linewidth]{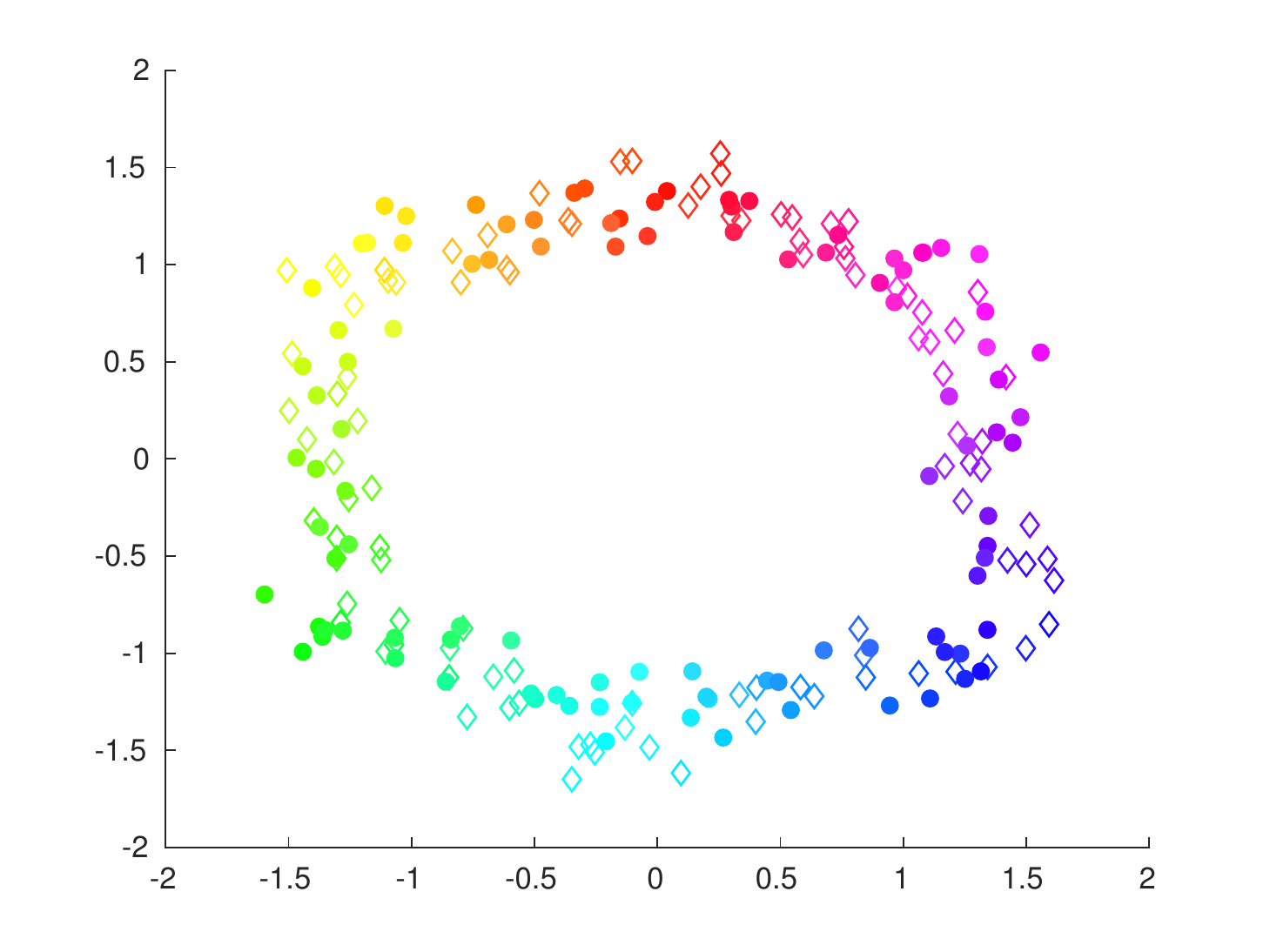} \\ 
  (a) \hspace{3.25cm} (b) \hspace{3.25cm}  (c)
\end{tabular}          
\end{center}
\caption{This figure shows the realignment results of CCA; (a) we consider  100 examples sampled from an ``arc'', each sample is endowed with an RGB feature vector. We duplicate this dataset using a 2D rotation (with an angle of $180^0$) and we add a random perturbation field both to the color features and 2D coordinates. (b) realignment results obtained using standard CCA; note that original data are not aligned, so in order to apply standard CCA, each sample in the first arc-set is paired with its nearest (color descriptor) neighbor in the second arc-set. (c) realignment results obtained using our AA CCA approach; again data are not paired, so we consider a fully dense matrix $\D$ that measures the cross-similarity (using an RBF kernel) between the colors of the first and the second arc-sets. In these toy experiments, $\beta$ (weight of context regularizer) is set to $0.01$ and  we use an isotropic neighborhood system in order to fill the context matrices $\{{\bf W}^c_u\}_{c=1}^{C}$, $\{{\bf W}^c_v\}_{c=1}^{C}$ (with $C=1$) and a given entry ${\bf W}^c_{u,i,k}$ is set to 1 iff  $\u_k$ is among the 10  spatial neighbors of $\u_i$. Similarly, we set the entries of  $\{{\bf W}^c_v\}_c^{C}$. {\bf For a better visualization of these results, better to view/zoom the PDF of this paper.}}\label{fig:toy}
\end{figure*}

\subsection{Satellite Image Change Detection}

We also evaluate and compare the performance of our proposed AA CCA method on the challenging task of satellite image change detection~(see for instance \cite{sahbi2015discriminant,RadkeAAR05,bourdis2011constrained,hussain2013change,sahbiCTF17,bourdis2012camera,Celik10,KunchevaF12,sahbi2013relevance}).  The goal is to find instances of relevant changes into a given scene acquired at instance $t_1$ with respect to the same scene taken at instant $t_0< t_1$; these acquisitions (at instants $t_0$, $t_1$) are referred to as reference and test images respectively. This task is known to be very challenging due to the difficulty to characterize relevant changes (appearance or disappearance of objects\footnote{This can be any object so there is no a priori knowledge about what object may appear or disappear into a given scene.}) from irrelevant ones such as the presence of cars, clouds, as well as {\it registration errors}. This task is also practically important; indeed, in the particular important scenario of damage assessment after natural hazards (such as tornadoes, earth quakes, etc.), it is crucial to achieve automatic change detection accurately in order to organize and prioritize rescue operations.
\subsubsection{JOPLIN-TORNADOES11 Dataset}
This dataset includes 680928 non overlapping image patches (of 30 $\times$ 30 pixels in RGB) taken from six pairs of (reference and test) GeoEye-1 satellite images (of 9850 $\times$ 10400 pixels each). This dataset is randomly split into two subsets: {\it labeled} used for training\footnote{From which a subset of 1000 is used for validation (as a dev set).} (denoted ${\cal L}_r \subset \S_r$, ${\cal L}_t \subset \S_t$) and {\it unlabeled} used for testing (denoted  ${\cal U}_r=\S_r \backslash {\cal L}_r$ and ${\cal U}_t = \S_t \backslash {\cal L}_t$) with $|{\cal L}_r|=|{\cal L}_t|=3000$ and  $|{\cal U}_r|=|{\cal U}_t|=680928-3000$. All patches in $\S_r$ (or in $\S_t$), stitched together, cover a very large area – of about 20 $\times$ 20 km$^2$ – around Joplin (Missouri) and show many changes after tornadoes that happened in may 2011 (building destruction, etc.) and no-changes (including irrelevant ones such as car appearance/disappearance, etc.).  Each patch in  $\S_r$, $\S_t$ is rescaled and encoded with 4096 coefficients corresponding to the output of an inner layer of the pretrained VGG-net \cite{Simonyan14c}. A given test patch is declared as a ``change'' or ``no-change'' depending on the score of SVMs trained on top of the learned CCA latent representations.\\
\noindent In order to evaluate the performances of change detection, we report the equal error rate (EER). The latter  is a balanced generalization error that equally weights errors in ``change'' and ``no-change'' classes. Smaller EER implies better performance.
\subsubsection{Data Pairing and Context Regularization}\label{pairing} 
In order to study the impact of AA CCA on the performances of change detection -- both with residual and relatively stronger misalignments -- we consider the following settings for comparison (see also table.~\ref{tablecca}). 
\begin{itemize}
\item {\bf Standard CCA:} patches are {\it strictly} paired by assigning each patch, in the reference image, to a unique patch in the test image (in the same location), so it assumes that satellite images are correctly registered. CCA learning is {\it supervised} (only labeled patches are used for training) and {\it no-context} regularization is used (i.e, $\beta=0$). In order to implement this setting,  we consider $\D$ as a diagonal matrix with $\D_{ii}=\pm 1$ depending on whether $\v_i \in {\cal L}_t$ is labeled as ``no-change'' (or ``change'') in the ground-truth, and $\D_{ii}=0$ otherwise.
\item {\bf Sup+CA CCA:} this is similar to  ``standard CCA''with the only difference being $\beta$ which is set to its ``optimal'' value ($0.01$) on the validation set (see Fig.~\ref{fig:beta}). 
\item {\bf SemiSup CCA:} this setting is similar to ``standard CCA'' with the only difference being the unlabeled patches which are now added when learning the CCA transformations, and $\D_{ii}$ (on the unlabeled patches) is  set to $2 \kappa(\v_i,\u_i)-1$ (score between $-1$ and $+1$); here $\kappa(.,.) \in [0,1]$ is the RBF similarity whose scale is set to the 0.1 quantile of pairwise distances in ${\cal L}_t \times {\cal L}_r$.
\item {\bf SemiSup+CA CCA:} this setting is similar to ``SemiSup CCA'' but context regularization is used (with again $\beta$ set to $0.01$). 
\item {\bf Res CCA:} this is similar to ``standard CCA'', but strict data pairing is {\it relaxed}, i.e., each patch in the reference image is assigned to multiple patches in the test image; hence, $\D$ is no longer diagonal, and set as $\D_{ij} = \kappa(\v_i,\u_j) \in [0,1]$ iff $(\v_i,\u_j) \in {\cal L}_t \times {\cal L}_r$ is labeled as ``no-change'' in the ground-truth,  $\D_{ij} =\kappa(\v_i,\u_j)-1 \in [-1,0]$ iff $(\v_i,\u_j)  \in {\cal L}_t \times {\cal L}_r$ is labeled as ``change'' and  $\D_{ij} =0$ otherwise.
 \item {\bf Res+Sup+CA CCA:} this is similar to  ``Res CCA''with the only difference being $\beta$ which is again set to $0.01$. 
\item {\bf Res+SemiSup CCA:} this setting is similar to ``Res CCA'' with the only difference being the unlabeled patches which are now added when learning the CCA transformations; on these unlabeled patches $\D_{ij} = 2\kappa(\v_i,\u_j)-1$. 
\item {\bf Res+SemiSup+CA CCA:} this setting is similar to ``Res+SemiSup CCA'' but context regularization is used (i.e., $\beta=0.01$). 
\end{itemize} 

\begin{figure}[t]
\begin{center}
\includegraphics[width=0.55\linewidth]{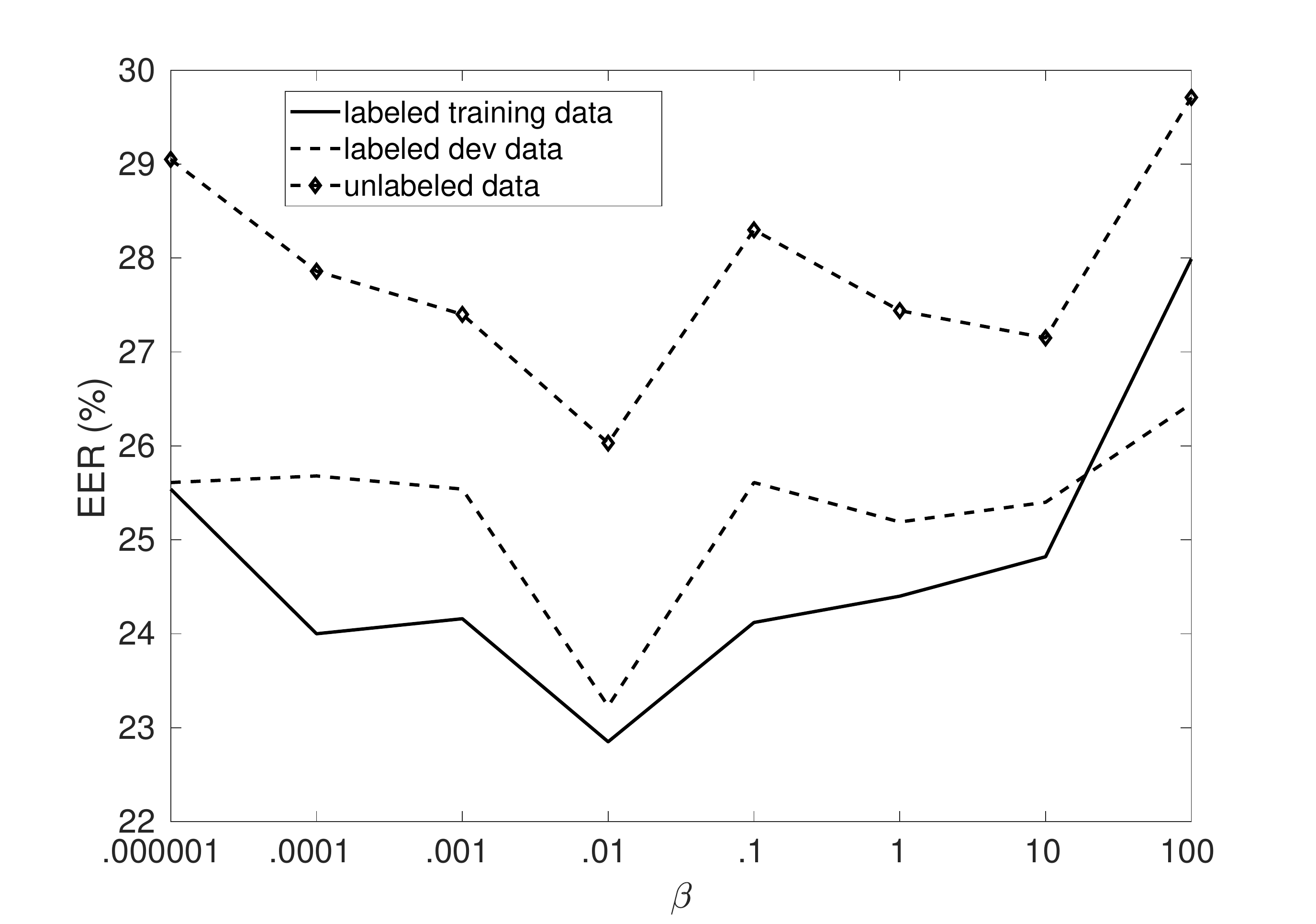}
\end{center}
\caption{This figure shows the evolution of change detection performances w.r.t $\beta$ on labeled training/dev data as well as the unlabeled data. These results correspond to the baseline Sup+CA CCA (under the regime of strong misalignments); we observe from these curves that $\beta=0.01$ is the best setting which is kept in all our experiments.}\label{fig:beta}
\end{figure}

\noindent {\bf Context setting:} in order to build the adjacency matrices of the context (see section~\ref{context}), we define for each patch $\u_i \in \S_r$ (in the reference image) an anisotropic (typed) neighborhood system $\{{\cal N}_c(i)\}_{c=1}^C$ (with $C=8$) which corresponds to the eight spatial neighbors of $\u_i$ in a regular grid \cite{thiemert2005applying,JeonM04,thiemert2006using}; for instance when $c=1$, ${\cal N}_1(i)$ corresponds to the top-left neighbor of $\u_i$. {Using  $\{{\cal N}_c(.)\}_{c=1}^8$, we build for each $c$  an intrinsic adjacency matrix ${\bf W}^c_u$  whose $(i,k)^{\textrm{th}}$ entry is set as ${\bf W}^c_{u,i,k} \propto \mathds{1}_{\{ k \in {\cal N}_c(i)\}}$; here $\mathds{1}_{\{\}}$ is the indicator function  equal to $1$ iff i) the patch $\u_k$ is neighbor to $\u_i$ and ii) its relative position is typed as $c$ ($c=1$ for top-left, $c=2$ for left, etc. following an anticlockwise rotation), and $0$ otherwise}. Similarly, we define  the matrices $\{{\bf W}^c_v\}_c$ for data $\{\v_j\}_j\in \S_t$.
\begin{table}
\begin{center}
  \resizebox{0.85\columnwidth}{!}{
\begin{tabular}{c|c|c||c}
Pairing & CCA Learning &  Context Regularization  & Designation \\
  \hline
  \hline
 strict & supervised  &  no  & Standard CCA \\
  strict & supervised  &  yes & Sup+CA CCA \\
 strict & semi-sup    &  no  & SemiSup CCA\\
 strict & semi-sup    &  yes & SemiSup+CA CCA \\
  \hline
relaxed & supervised  & no   & Res CCA\\
relaxed & supervised  & yes  & Res+Sup+CA CCA \\
relaxed & semi-sup    & no   & Res+SemiSup CCA \\
relaxed & semi-sup    & yes  & Res+SemiSup+CA CCA
\end{tabular}}
\vspace{0.15cm}
\caption{This table shows different configurations of CCA resulting from different instances of our model. In this table,  ``Sup'' stands for supervised, ``SemiSup'' for semi-supervised, ``CA'' for context aware and ``Res'' for resilient.}\label{tablecca}
\end{center}
\end{table} 
\subsubsection{Impact of AA CCA and Comparison} 
Table.~\ref{tableres} shows a comparison of different versions of AA CCA against other CCA variants under the regime of small residual alignment errors. In this regime, reference and test images are first registered using RANSAC~\cite{kim2003automatic}; an exhaustive visual inspection of the overlapping (reference and test) images (after RANSAC registration) shows sharp boundaries in most of the areas covered by these images, but some areas still include residual misalignments due to the presence of changes, occlusions (clouds, etc.) as well as parallax. Note that, in spite of the relative success of RANSAC in registering these images, our AA CCA versions (rows \#5--8) provide better performances (see table.~\ref{tableres}) compared to the other CCAs (rows \#1--4); this clearly corroborates the fact that residual alignment errors remain after RANSAC (re)alignment (as also observed during visual inspection of RANSAC registration). Put differently, our AA CCA method is not an opponent to RANSAC but complementary. \\
\begin{table}
  \begin{center}
\resizebox{0.80\columnwidth}{!}{
\begin{tabular}{cc||c|c|c}
\# &  Configurations & Labeled(train) & Labeled(dev) & Unlabeled \\
  \hline
  \hline
  1& Standard CCA   & 14.91 & 15.18 & 12.81 \\
  2& Sup+CA CCA     &   12.95 & 14.90 & 11.44 \\
  \hline 
  3& SemiSup CCA        &   11.26 & 12.80 & 11.18 \\
  4& SemiSup+CA CCA     &   12.57 & 11.82 & 09.96 \\
  \hline
  \hline 
  5& Res CCA            &   05.81 & 04.97 & {\bf 05.38} \\
  6& Res+Sup+CA CCA     &   06.35 & 05.53 & {\bf 05.55} \\
  \hline 
  7& Res+SemiSup CCA    &   08.60 & 08.74 & 08.33 \\
  8& Res+SemiSup+CA CCA &   08.77 & 08.60 & 06.94
\end{tabular}}
\vspace{0.15cm}
\caption{This table shows change detection EER (in \%) on labeled (training and validation) and unlabeled sets under the residual error regime. When context regularization (referred to as CA in this table) is used, $\beta$ is set to $10^{\small -2}$.}\label{tableres}
\end{center}
\end{table} 
\noindent These results also show that when reference and test images are globally well aligned (with some residual errors; see table.~\ref{tableres}), the gain in performance is dominated by the positive impact of alignment resilience; indeed, the impact of the unlabeled data is not always consistent (\#5,6 vs. \#7,8 resp.) in spite of being positive (in \#1,2 vs. \#3,4 resp.)  while the impact of context regularization is globally positive (\#1,3,5,7 vs. \#2,4,6,8 resp.). This clearly shows that, under the regime of small residual errors, the use of labeled data is already enough in order to enhance the performance of change detection; the gain comes essentially from alignment resilience with a marginal (but clear) positive impact of context regularization. \\ 
\indent In order to study the impact of AA CCA w.r.t stronger alignment errors (i.e. w.r.t a more challenging setting), we apply a relatively strong motion field to all the pixels in the reference image; precisely, each pixel is shifted along a direction whose x--y coordinates are randomly set to values between 15 and 30 pixels. These shifts are sufficient in order to make the quality of alignments used for CCA very weak so the different versions of CCA, mentioned earlier, become more sensitive to alignment errors (EERs increase by more than 100\% in table.~\ref{tablestrong} compared to EERs with residual alignment errors in table.~\ref{tableres}). With this setting, AA CCA is clearly more resilient and shows a substantial relative gain compared to the other CCA versions.
\begin{table}
\begin{center}
  \resizebox{0.80\columnwidth}{!}{
\begin{tabular}{cc||c|c|c}
\# &  Configurations & Labeled(train) & Labeled(dev) & Unlabeled \\
  \hline
  \hline
  1& Standard CCA    
                     & 25.63 & 25.61 & 28.44 \\
  2& Sup+CA CCA   
  
                     & 22.85 & 23.23 & {26.03} \\ 
\hline
  3& SemiSup CCA  
                     & 22.31 & 23.58 & 24.99 \\
  4& SemiSup+CA CCA 
                     & 25.74 & 25.40 & {25.47}  \\ 
  \hline
  \hline 
  5& Res CCA       
                     & 16.42 & 14.34 & {\bf 19.67} \\
  6& Res+Sup+CA CCA 
                     & 16.55 & 16.80 & {\bf 19.90}  \\ 
  \hline 
  7& Res+SemiSup CCA 
                     & 19.01 & 19.24 & {\bf 19.55}  \\
  8& Res+SemiSup+CA CCA
                     & 23.71 & 21.55 & 26.76                             
\end{tabular}}
\vspace{0.15cm}
\caption{This table shows change detection EER (in \%) on labeled (training and validation) and unlabeled sets under the strong error regime. When context regularization (referred to as CA in this table) is used, $\beta$ is set to $10^{-2}$.}\label{tablestrong}
\end{center}
\end{table}
\subsection{Discussion}
\noindent {\bf Invariance:} resulting from its misalignment resilience, it is easy to see that our AA CCA is {\it de facto} robust to local deformations as these deformations are strictly equivalent to local misalignments. It is also easy to see that our AA CCA may achieve invariance to similarity transformations; indeed, the matrices used to define the spatial context are translation invariant, and can also be made rotation and scale invariant  by measuring a ``characteristic'' scale and orientation of  patches in a given satellite image. For that purpose, dense SIFT can be used to recover (or at least approximate) the field of orientations and scales, and hence adapt the spatial support (extent and orientation) of context using the characteristic scale, in order to make context invariant to similarity transformations. \\

\noindent {\bf Computational Complexity:} provided that VGG-features are extracted (offline) on all the patches of the reference/test images, and provided that the adjacency matrices of context are precomputed\footnote{Note that the adjacency matrices of the spatial  neighborhood system can be computed offline once and reused.}, and since the adjacency matrices  $\{{\bf W}^c_u\}_c$, $\{{\bf W}^c_v\}_c$ are very sparse,  the computational complexity of evaluating Eq.~(\ref{ep001}) and solving the generalized eigenproblem in ~Eq.~(\ref{ep01}) both reduce to  $O(\min(d_u^2 d_v,d_v^2 d_u))$, here $d_u$, $d_v$ are again the dimensions of data in $\bU$, $\bV$ respectively; hence, this complexity is very equivalent to standard CCA which also requires solving a generalized eigenproblem. Therefore, the gain in the accuracy of our AA CCA  is obtained without any overhead in the computational complexity   that remains dependent on dimensionality of data (which is, in practice, smaller compared to the cardinality of our datasets). 

\begin{figure*}
  \centering
  \begin{tabular}[b]{c}
       {\includegraphics[width=0.3\linewidth]{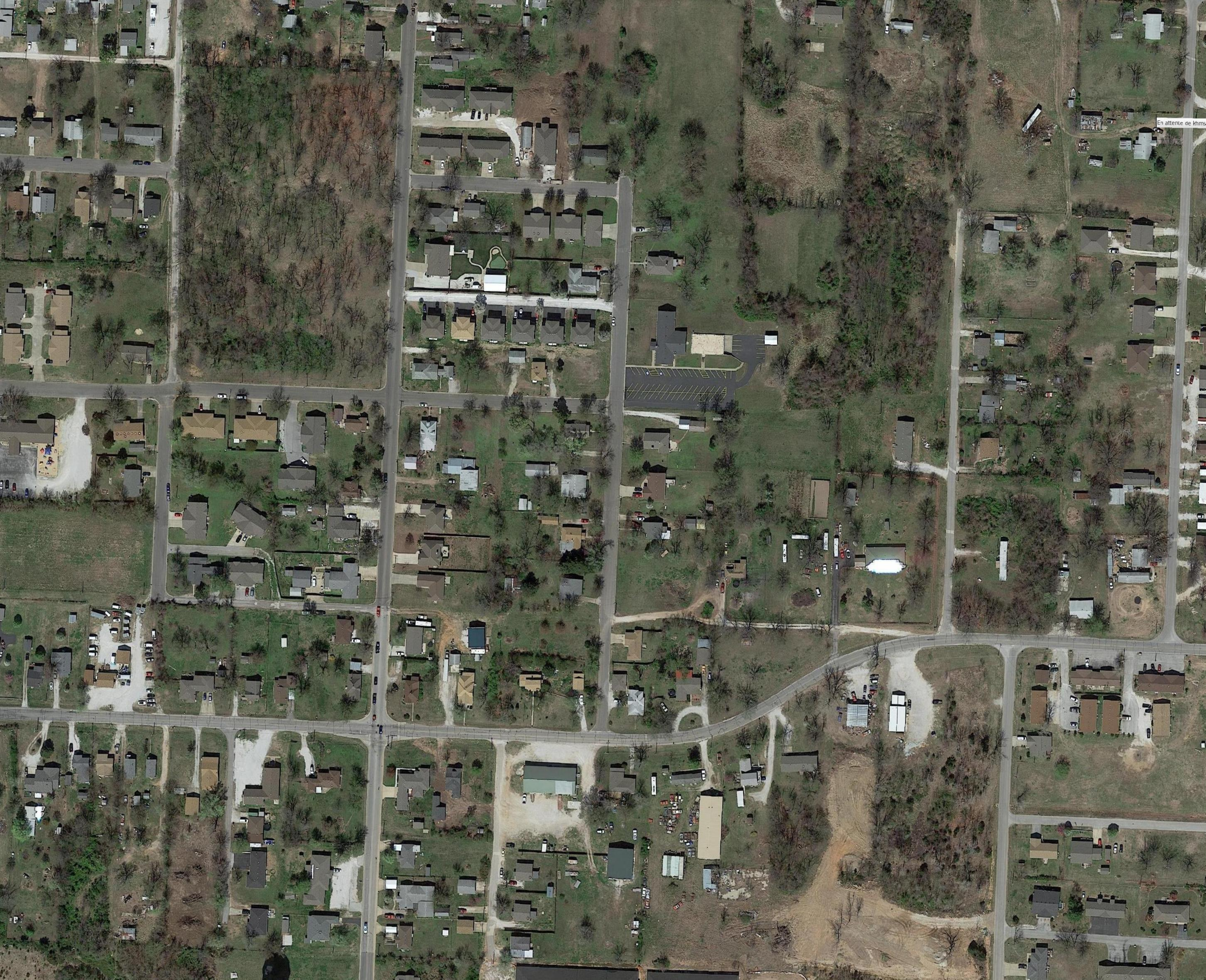}} \\
       \small  Ref image
  \end{tabular} 
    \begin{tabular}[b]{c}
       {\includegraphics[width=0.3\linewidth]{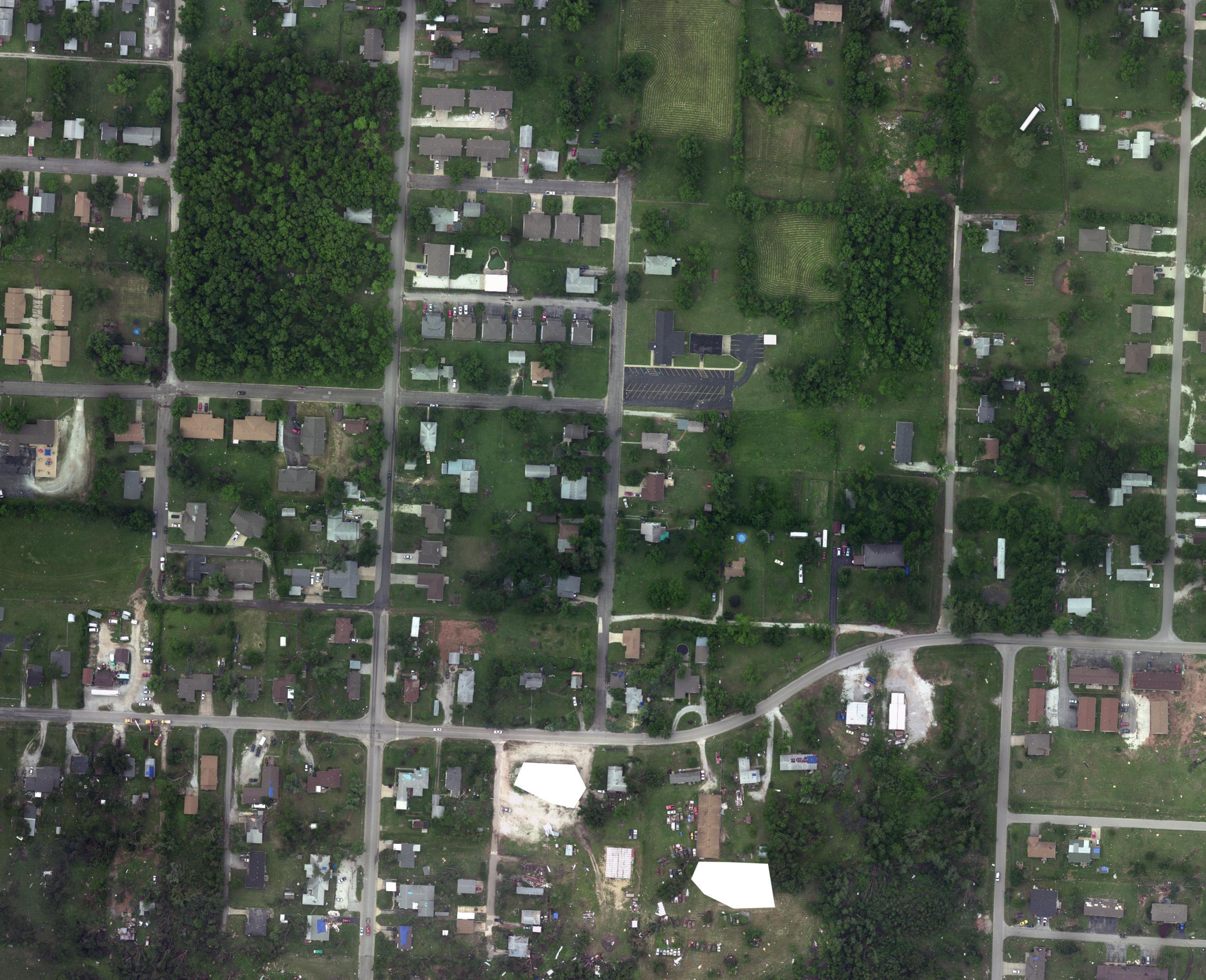}} \\ 
      \small  Test image + GT mask
    \end{tabular}    
 \begin{tabular}[b]{c}
       {\includegraphics[width=0.3\linewidth]{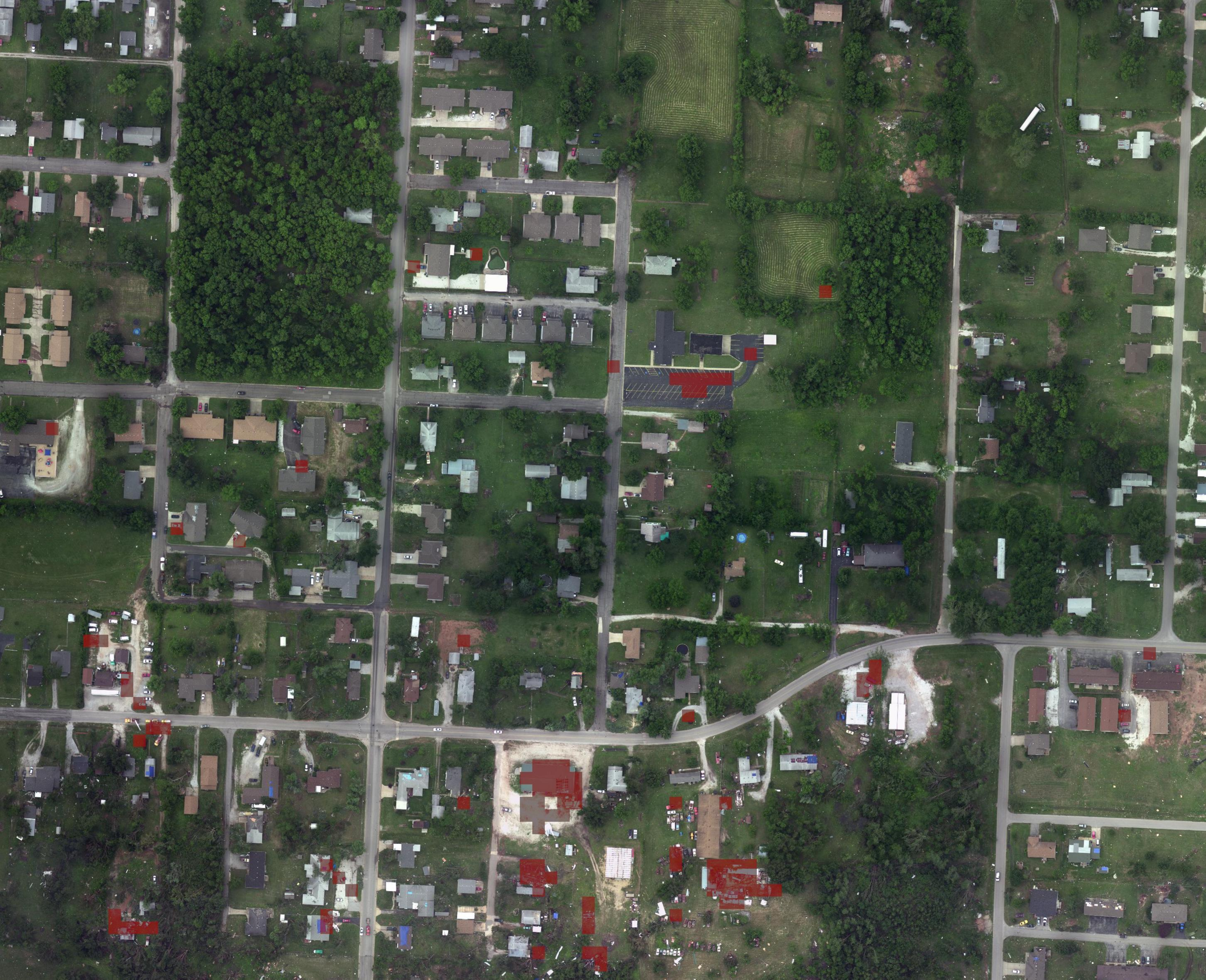}} \\
       \small  Standard CCA
 \end{tabular}
       \begin{tabular}[b]{c}
       {\includegraphics[width=0.3\linewidth]{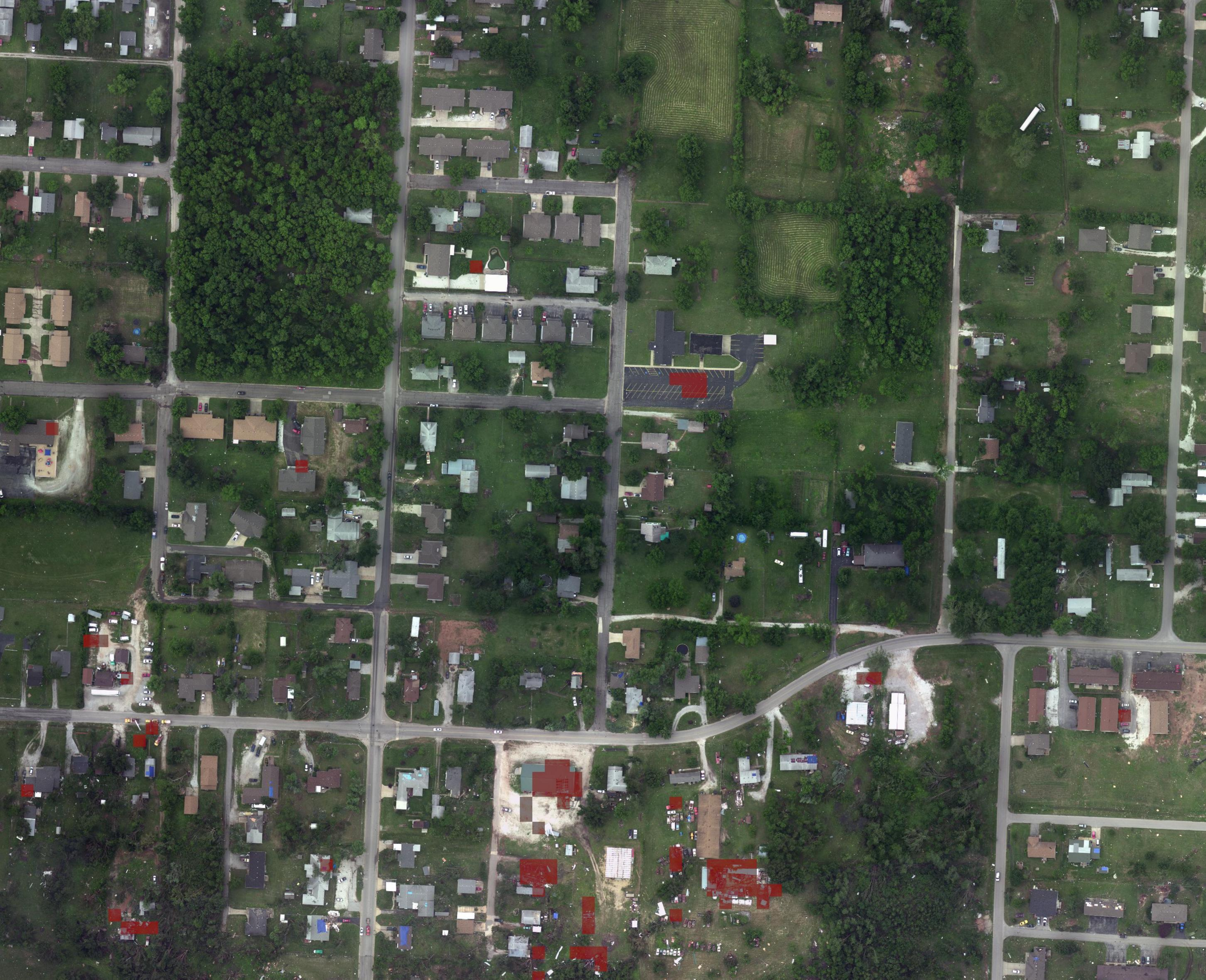}} \\
       \small  Sup CCA+CA
  \end{tabular} 
    \begin{tabular}[b]{c}
       {\hspace{-0.05cm} \includegraphics[width=0.3\linewidth]{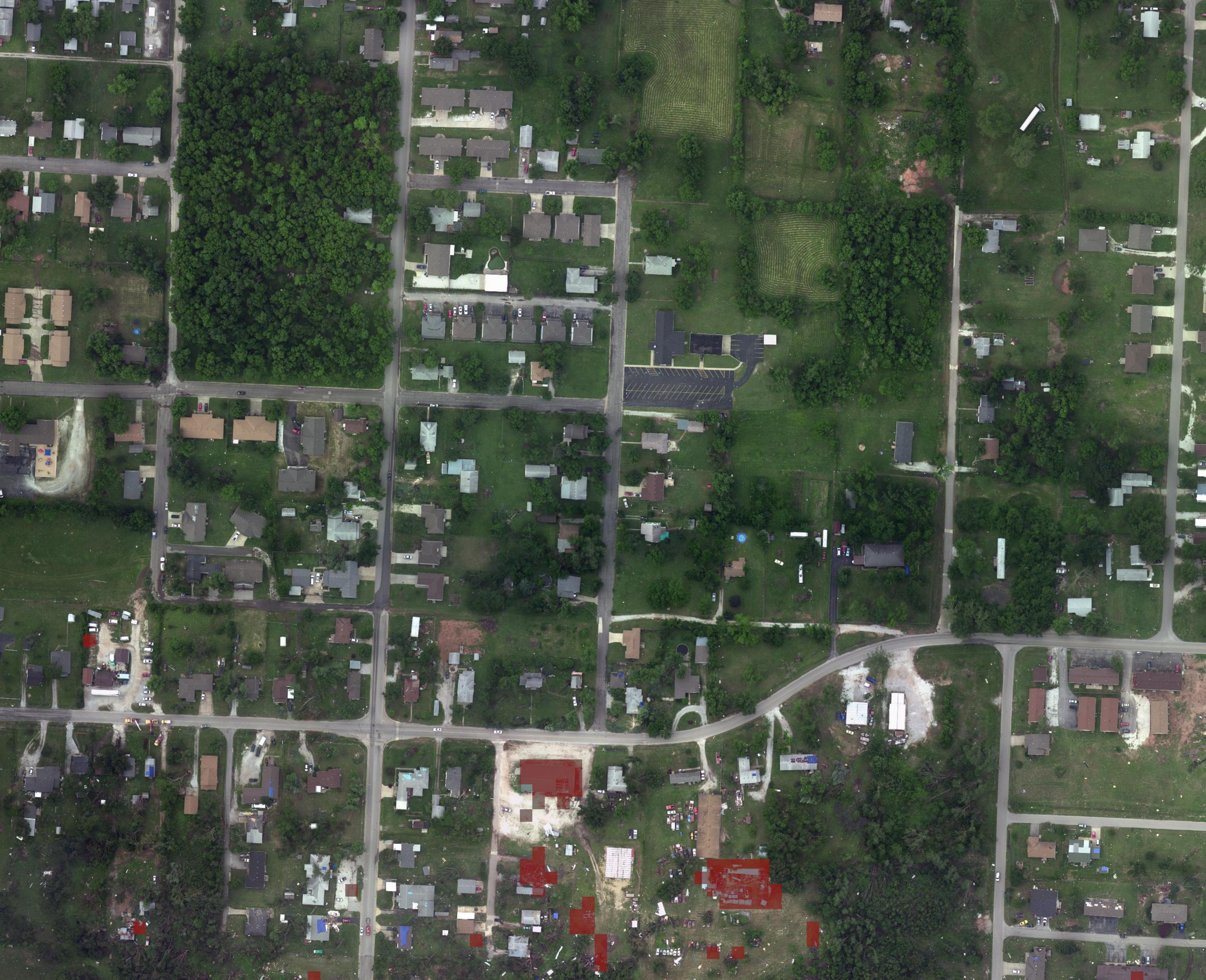}} \\ 
      \small  Res CCA  
    \end{tabular}    
 \begin{tabular}[b]{c}
       {\hspace{0.00cm}\includegraphics[width=0.3\linewidth]{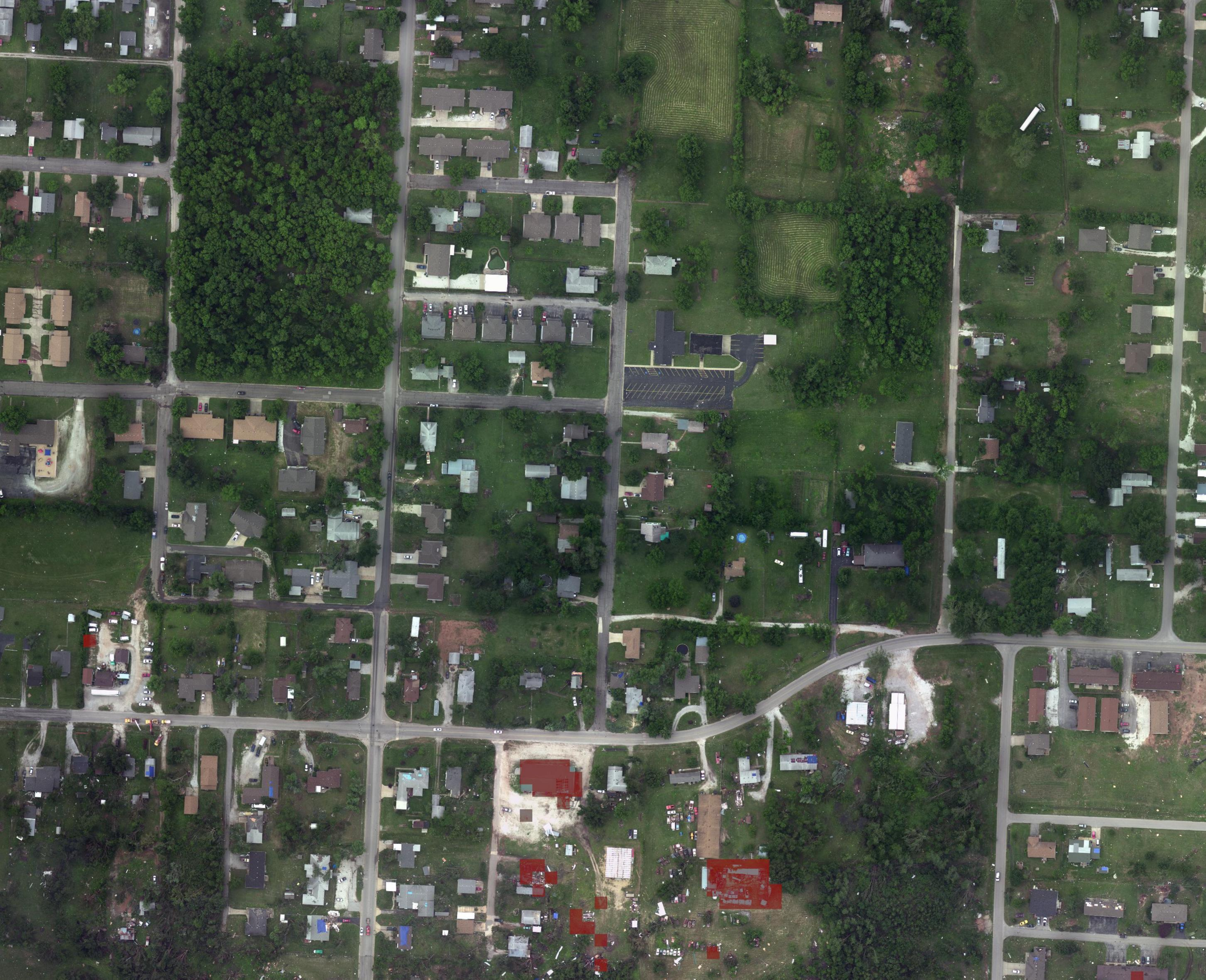}} \\
       \small  Res CCA+CA
 \end{tabular}
 \caption{These examples  show the evolution of detections (in red) for four different settings of CCA; as we go from top-right to bottom-right, change detection results get better. CCA acronyms shown below pictures are already defined in Table.1.}\label{fig:sat}
\end{figure*}

\section{Conclusion}\label{section4}
We introduced in this paper a new canonical correlation analysis method that learns projection matrices which map data from input spaces to a latent common space where unaligned data become strongly or weakly correlated depending  on  their  cross-view similarity and their context.  This  is  achieved  by  optimizing a criterion that mixes  two  terms:  the  first one aims at maximizing the correlations between data which are likely to be paired while the second  term  acts  as  a  regularizer and makes correlations spatially  smooth  and  provides  us  with  robust  context-aware latent representations. Our method  considers both labeled and unlabeled data when learning the CCA projections while being resilient  to  alignment  errors.  Extensive experiments show the substantial gain of our CCA method under the regimes of residual and strong alignment errors.

As a future work, our CCA method can be extended  to many  other tasks  where  alignments  are error-prone  and  when  context  can  be  exploited  in  order to  recover  from  these  alignment  errors.  These  tasks  include ``text-to-text'' alignment  in  multilingual  machine  translation,  as  well  as ``image-to-image'' matching in multi-view object tracking.
\appendix
\section{Appendix (proof of Proposition 2)} 
{\normalfont
  \small
\noindent We will prove that $\Psi$ is $L$-Lipschitzian, \\ with $L=\frac{\beta}{\gamma_{\min}} \big(\sum_c \big\|  {\bf E}_c \ {\bf 1}_{\tiny vu} \ {\bf F}_c' \big\|_1 +  \sum_c \big\|{\bf G}_c \ {\bf 1}_{\tiny vu} \ {\bf H}_c' \big\|_1\big)$. For ease of writing,  we omit in this proof the subscripts $t$, $r$ in ${\bf K}_{tr}$ (unless explicitly required and mentioned).  \\ 
Given two matrices ${\bf K}^{(2)}$, ${\bf K}^{(1)}$, we have $\big\| {\bf K}^{(2)} - {\bf K}^{(1)} \big\|_1=(*)$ with   \\ 
\begin{equation}\label{eq233}
  \begin{array}{ll}
(*) &=  \displaystyle \beta \bigg\|  \sum_c \bV {\bf W}^c_v \bV' (\Pt^{(1)} \Prt^{(1)} -\Pt^{(0)} \Prt^{(0)}) \bU {\bf W}^{c'}_u \bU' \\
 &    \ \ \ \    + \displaystyle \ \sum_c \bV {\bf W}^{c'}_v \bV' (\Pt^{(1)} \Prt^{(1)}-\Pt^{(0)} \Prt^{(0)}) \bU {\bf W}^{c}_u \bU' \bigg\|_1.
\end{array}
\end{equation}
Using Eq.~(\ref{ep01}), one may write 
\begin{equation}\label{hh00}
\Pt \Prt = \frac{1}{\gamma}  \C_{tt}^{-1}  {\bf K}_{tr} \C_{rr}^{-1},\end{equation} 
 which also results from the fact that ${\bf K}_{rt}  \C_{tt}^{-1}  {\bf K}_{tr}$ is Hermitian and $\C_{rr}$ is positive semi-definite. By adding the superscript $\tau$ in $\Pt$, $\Pr$, $\gamma$, ${\bf K}_{tr}$ (with $\tau=0,1$), omitting again the subscripts  $t$, $r$ in  ${\bf K}_{tr}$ and then plugging (\ref{hh00}) into (\ref{eq233}) we obtain
\begin{equation}\label{eq234}
\small
\begin{array}{lll}
(*) &= & \displaystyle \beta \bigg\|  \sum_c {\bf E}_c  \ (\frac{1}{\gamma^{(1)}}{\bf K}^{(1)}-\frac{1}{\gamma^{(0)}}{\bf K}^{(0)}) \   {\bf F}'_c    + \displaystyle \ \sum_c {\bf G}_c \ (\frac{1}{\gamma^{(1)}}{\bf K}^{(1)}-\frac{1}{\gamma^{(0)}}{\bf K}^{(0)}) \  {\bf H}_c' \bigg\|_1 \\  
 & \leq & \displaystyle \frac{\beta}{\gamma_{\min}}  \bigg\|  \sum_c {\bf E}_c  \ ({\bf K}^{(1)}-{\bf K}^{(0)})  \ {\bf F}_c'  \bigg\|_1   +  \ \displaystyle  \frac{\beta}{\gamma_{\min}}  \bigg\|  \sum_c 
{\bf G}_c  \ ({\bf K}^{(1)}-{\bf K}^{(0)}) \ {\bf H}_c' \bigg\|_1,  
\end{array}
\end{equation}
\noindent here $\gamma_{\min}$ is the lower bound of the eigenvalues of (\ref{ep01}) which can be derived (see for instance \cite{lu2000some}). Considering ${\bf K}_{k,\ell}$ as the $(k,\ell)^{\rm{th}}$ entry of ${\bf K}$, we have
\begin{equation*}
\small
\begin{array}{lll}
(*) & \leq  & \ \ \frac{\beta}{\gamma_{\min}}   \displaystyle \sum_{i,j} \bigg| \displaystyle \sum_{k,\ell} ({\bf K}^{(1)}_{k,\ell} - {\bf K}^{(0)}_{k,\ell})  \sum_{c} {\bf E}_{c,i,k} \  {\bf F}_{c,j,\ell}     \bigg| \\ 
 &  & \ \ \  +  \frac{\beta}{\gamma_{\min}}   \displaystyle \sum_{i,j} \bigg| \displaystyle \sum_{k,\ell} ({\bf K}^{(1)}_{k,\ell} - {\bf K}^{(0)}_{k,\ell})  \sum_{c} {\bf G}_{c,i,k} \  {\bf H}_{c,j,\ell}     \bigg| \\
 & & \\
  & {\bf \leq} & \ \ \frac{\beta}{\gamma_{\min}} \displaystyle \sum_{i,j} \displaystyle \sum_{k,\ell} \big| {\bf K}^{(1)}_{k,\ell} - {\bf K}^{(0)}_{k,\ell}  \big| \sum_c \big|  {\bf E}_{c,i,k} \  {\bf F}_{c,j,\ell}      \big| \\
  &  &  \ \ \ +\frac{\beta}{\gamma_{\min}} \displaystyle \sum_{i,j} \displaystyle \sum_{k,\ell} \big| {\bf K}^{(1)}_{k,\ell} - {\bf K}^{(0)}_{k,\ell}  \big| \sum_c \big|  {\bf G}_{c,i,k} \  {\bf H}_{c,j,\ell}      \big| \\
&  &   \\
  & \leq  & \ \ \frac{\beta}{\gamma_{\min}} \displaystyle \sum_{k,\ell} \big| {\bf K}^{(1)}_{k,\ell} - {\bf K}^{(0)}_{k,\ell} \big| \\
 &   &   \ \times  \bigg(\displaystyle \sum_{i,j} \sum_{k,\ell,c} \big|  {\bf E}_{c,i,k} \  {\bf F}_{c,j,\ell}   \big|  + \displaystyle \sum_{i,j} \sum_{k,\ell,c} \big|  {\bf G}_{c,i,k} \  {\bf H}_{c,j,\ell}      \big|\bigg)   \\
 &   &  (\textrm{as} \   \sum_{i} |a_i|.|b_i| \leq \sum_{i,j} |a_i|.|b_j|, \ \forall \ \{a_i\}_i, \{b_j\}_i \subset \mathbb{R} )\\
  &    & \\ 
 & = &   L \ \big\| {\bf K}^{(1)} - {\bf K}^{(0)} \big\|_1,
\end{array}  
\end{equation*} 
\begin{equation*}
\small
\textrm{with} \ \ \ L=\frac{\beta}{\gamma_{\min}} \big( \sum_c \big\| {\bf E}_c \ {\bf 1}_{\tiny vu} \ {\bf F}_c' \big\|_1 + \sum_c  \big\| {\bf G}_c \ {\bf 1}_{\tiny vu} \ {\bf H}_c' \big\|_1\big)\Box 
\end{equation*}
}

\bibliographystyle{IEEEtran}
\bibliography{egbib}

\end{document}